\documentclass[conference]{IEEEtran}
\IEEEoverridecommandlockouts
\usepackage{cite}
\usepackage{amsmath,amssymb,amsfonts,amsthm}
\newtheorem{theorem}{Theorem}

\usepackage{algorithmic}
\usepackage{graphicx}
\usepackage{textcomp}
\usepackage{xcolor}
\usepackage{wrapfig}

\def\BibTeX{{\rm B\kern-.05em{\sc i\kern-.025em b}\kern-.08em
    T\kern-.1667em\lower.7ex\hbox{E}\kern-.125emX}}
\begin{document}

\title{End-to-End Stable Imitation Learning via Autonomous Neural Dynamic Policies}

\author{Dionis~Totsila$^{1\dagger *}$,
        Konstantinos~Chatzilygeroudis$^{2\dagger *}$,
        Denis~Hadjivelichkov$^3$,
        Valerio~Modugno$^3$,\\
        Ioannis~Hatzilygeroudis$^1$,
        and~Dimitrios~Kanoulas$^3$
\thanks{$^*$Corresponding~authors:~{\tt\small dtotsila@upnet.gr \newline costashatz@upatras.gr }}
\thanks{$^\dagger$ Equal Contribution}
\thanks{$^1$\textit{Computer Engineering and Informatics Department (CEID), University of Patras, Greece}}
\thanks{$^2$\textit{CILab, Department of Mathematics, University of Patras, Greece}}
\thanks{$^3$\textit{RPL Lab, University College London, United Kingdom}}
}
\maketitle

\begin{abstract}
State-of-the-art sensorimotor learning algorithms offer policies that can often produce unstable behaviors, damaging the robot and/or the environment. Traditional robot learning, on the contrary, relies on dynamical system-based policies that can be analyzed for stability/safety. Such policies, however, are neither flexible nor generic and usually work only  with proprioceptive sensor states. In this work, we bridge the gap between generic neural network policies and dynamical system-based policies, and we introduce Autonomous Neural Dynamic Policies (ANDPs) that: (a) are based on autonomous dynamical systems, (b) always produce asymptotically stable behaviors, and (c) are more flexible than traditional stable dynamical system-based policies. ANDPs are fully differentiable, flexible generic-policies that can be used in imitation learning setups while ensuring asymptotic stability. In this paper, we explore the flexibility and capacity of ANDPs in several imitation learning tasks including experiments with image observations. The results show that ANDPs combine the benefits of both neural network-based and dynamical system-based methods.
\end{abstract}


\section{Introduction}

Choosing the appropriate policy structure is crucial for effective and practical robot learning~\cite{varin2019comparison,chatzilygeroudis2019survey,martin2019variable}. Currently, either in the context of Reinforcement Learning (RL)~\cite{sutton1998reinforcement} or Imitation Learning (IL)~\cite{Rajeswaran2018RSS}, the standard choice is to use Neural Networks (NNs). NNs possess the flexibility needed to learn complicated behaviors as well as the generalizability required to be able to run the same algorithms in any robot or scenario.  However, NN-based policies are black-box and cannot ensure well-behaved trajectories, meaning that their behavior cannot be predicted in unforeseen situations. As a result, RL algorithms often make harmful decisions for the robot and/or  the environment, especially during initial learning stages.

This comes in contrast with traditional robot learning literature~\cite{billard2008robot}, where usually in the context of IL/Learning from Demonstrations (LfD) the policy behavior is shaped according to some well-defined criteria. For example, producing behaviors that are guaranteed to asymptotically converge to an attractor is an important concept of traditional robot learning~\cite{billard2008robot,khansari2011learning}. The main building tool for this is a Dynamical System (DS), and the main idea is to represent the policy to be learned as a DS. This gives us the ability to reason about the policy in terms familiar to control theory and make proofs about properties that we care about (e.g., for asymptotic stability)~\cite{khansari2011learning}. This concept has been explored in robotic scenarios with two main different approaches: (a) time-dependent DSs that mostly fall under the framework of Dynamic Movement Primitives (DMPs)~\cite{stulp2013robot,schaal2006dynamic,ijspeert2013dynamical}, and (b) autonomous DSs where the input is only dependent on the current state~\cite{khadivar2021learningdswithbigurcations, khansari2011learning}. Both approaches are able to provide asymptotic stability guarantees, while autonomous DSs' reactiveness does not depend on time. Traditionally, both approaches have been utilized mainly in Imitation Learning (IL)/Learning from Demonstrations (LfD) scenarios~\cite{ude2010task,ude2016trajectory,ijspeert2003learning,figueroa2018physically,khansari2011learning,amanhoud2019dynamical}, but recently there was an attempt to use DMPs with RL~\cite{bahl2020neural,bahl2021hierarchical}.
\begin{figure}[!t]
    \centerline{\includegraphics[width=\linewidth]{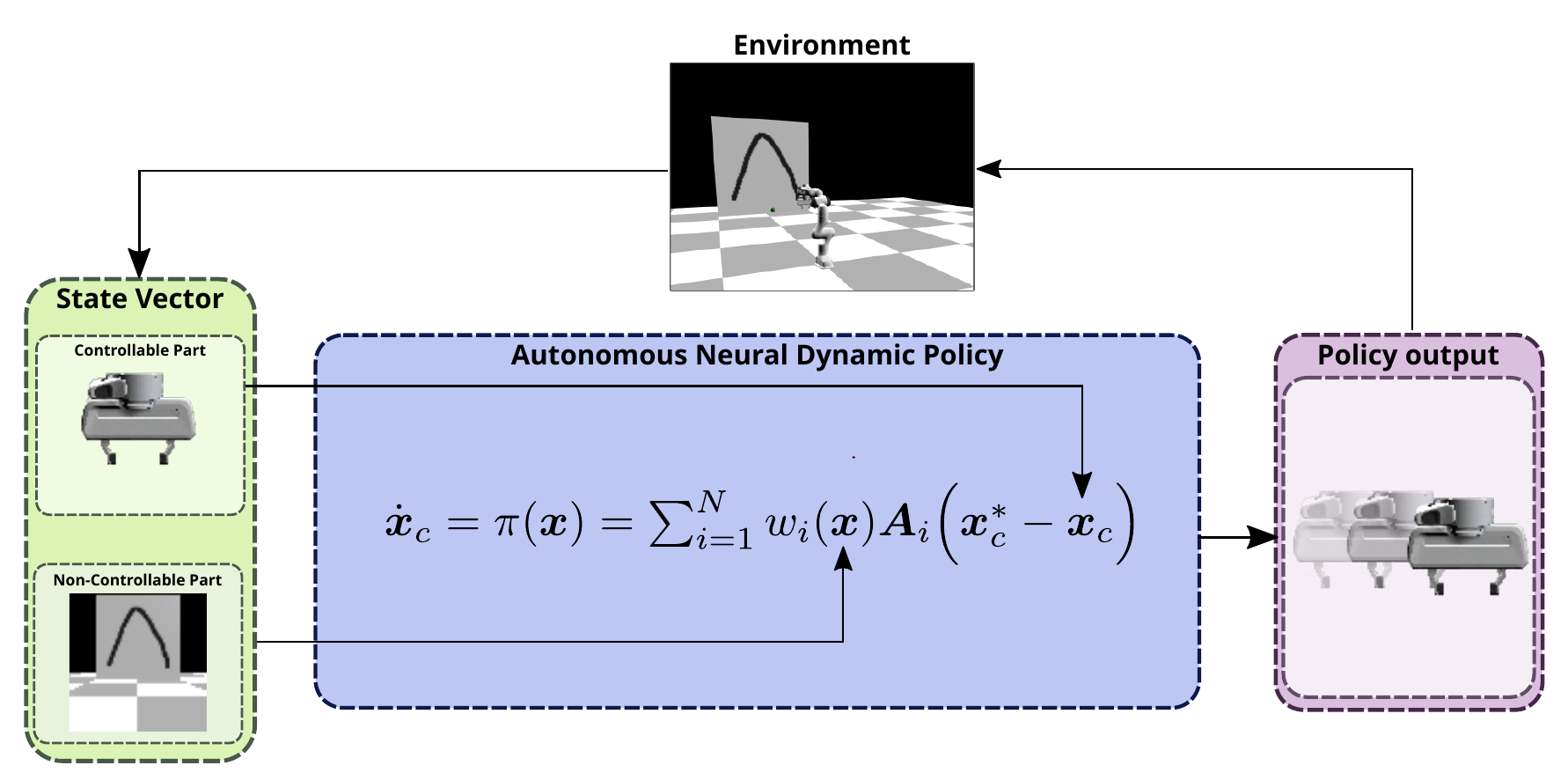}}
    \caption{Autonomous Neural Dynamic Policy Outline.}
    \label{fig:policy_overview}
    \vspace{-2em}
\end{figure}

In this work, we aim at developing a novel policy representation that:
\begin{enumerate}
    \setlength\itemsep{0em}
    \item Produces behaviors that are guaranteed to asymptotically converge to an attractor;
    \item Is universal, meaning that any robot/embodied agent can use it;
    \item Can accept any observation space (e.g. RGB images), and the input is not limited only to the feedback of robot states.
\end{enumerate}

The first feature guarantees that the policy does not produce unpredictable behaviors outside the trained data, but will try to go toward the attractor, which is a critical aspect for robotic applications. Even if this feature cannot guarantee that the robot will never try anything harmful (e.g., very big velocity command), it is a good stepping stone towards having safe robot behaviors. The second characteristic is important as it will make the policy usable by any robotic mechanism and any scenario. The last characteristic is important for practical robotic applications since traditionally policies that provide stability guarantees are not straightforward to utilize with states that are not limited to proprioceptive sensor states.

    

Inspired by the LfD literature and the usage of DSs, we attempt to integrate them with the representation capabilities of NNs, with the intention of creating a policy with the aforementioned properties. To achieve this goal:

\begin{itemize}
    \setlength\itemsep{0em}
    \item We assume that the state of the system can be divided into a part that is directly controlled (e.g., joint positions), and a part that can only be observed and/or indirectly controlled (e.g., a box);
    \item We represent the policy as a non-linear combination of linear systems; this gives us the ability to reason about the stability of the produced behavior, while also being able to exploit the representation abilities of NNs to combine the elementary DSs effectively.
\end{itemize}
We call this new type of policies Autonomous Neural Dynamic Policies (ANDPs) because they: (a) are based on dynamical systems, (b) are generic policies, (c) are based on neural networks, and (d) do not depend on the time (thus, autonomous). The main novelty of ANDPs lies in their ability to accept any input space (even images), while still producing asymptotically stable behaviors of the controllable state. This is achieved by imposing constraints on the elementary linear DSs, while making the policy expressive by using NNs to combine them. Finally, we perform some reparameterizations to ``eliminate'' the constraints and to be able to optimize ANDPs using gradient-based unconstrained optimization.
\section{Related Work}
%
%
The choice of the policy structure plays an important role for the effectiveness of learning in practical robot applications. When designing the policy structure, there is always a tradeoff between having a representation that is expressive, and one that provides a space that is efficiently searchable~\cite{chatzilygeroudis2019survey}.

The most obvious way to make the policy easy to be searched (or optimized) is to hand-design it. In~\cite{fidelman2004learning}, for example, the authors design by hand a policy for a ball acquisition task, which has only four parameters. This low-dimensional policy can be easily optimized, but with only four parameters it might not possess big expressiveness. Moreover, if we are faced with a different robot/task, we need to re-design the policy from scratch. On the other hand, using a function approximator (e.g. a neural network) to describe the policy, enables us to easily increase the expressiveness and generality of the policy, but can make the policy difficult to optimize for.

There are numerous works that describe best practices and policy structures that ease the learning process. In~\cite{martin2019variable} the authors thoroughly evaluate different action spaces (with the corresponding low-level controllers, and thus policy structures) in a wide range of scenarios. They conclude that operating in end-effector space combined with low-level controllers make the learning process faster, more robust, and provide easier transfer from simulation to the physical world and between robots. The authors in~\cite{varin2019comparison} reach to similar conclusions in a related study that performs a comparison between different action spaces for manipulation tasks. Overall, it is clear from the literature that in order to learn effectively in practical robotic applications, we need a structured policy representation.

Dynamic Movement Primitives (DMPs)~\cite{stulp2013robot,schaal2006dynamic,ijspeert2013dynamical} provide a framework for structured policy types that are basically a dynamical system. As such, we can insert desired properties that can make our system behave in specific ways. DMPs are split into two systems: (a) the canonical system (which is usually a springer-damper system), and (b) the transformation system. The canonical system represents the movement \emph{phase} $s$, which starts at 1 and converges to 0 over time. The transformation systems combine a spring-damper system with a function approximator (e.g., NNs) which, when integrated, generates accelerations. Multi-dimensional DMPs are achieved by coupling multiple transformation systems with one canonical system. DMPs can be used in the end-effector and the joint angle space. There are numerous successful implementations of DMPs mainly in IL/LfD scenarios covering a wide range of tasks~\cite{ijspeert2013dynamical,ijspeert2003learning,ude2016trajectory} and even multi-task cases~\cite{ude2010task,stulp13learning}.

DMPs, however, are time-dependent and thus they may produce undesirable behaviors; for example, a policy that cannot adapt to perturbations after some time. Stable Estimator of Dynamical Systems (SEDS)~\cite{khansari2011learning} explores how to use dynamical systems in order to define autonomous (i.e., time-independent) controllers (or policies) that are asymptotically stable. The main idea of the algorithm is to use a finite mixture of Gaussian functions (Gaussian Mixture Models - GMMs) as the policy, $\dot{\boldsymbol{\xi}} = \pi_{\text{seds}}(\boldsymbol{\xi})$, with specific properties that satisfy some stability guarantees. SEDS, however, requires demonstrated data in order to optimize the policy (i.e., data gathered from experts), although similar ideas have been used within the RL framework~\cite{guenter2007reinforcement}. SEDS and its variants~\cite{figueroa2018physically,shavit2018learning} have provided effective solutions to difficult tasks ranging from point-to-point motions~\cite{khansari2011learning} to humanoid navigation~\cite{figueroa2020dynamical} and following force profiles~\cite{amanhoud2019dynamical}. One of the main limitations of SEDS is the \emph{accuracy vs. stability dilemma}, i.e., it performs poorly in highly non-linear motions that contain high curvatures or that are non-monotonic. This is mainly because of the constraints SEDS imposes on the structure of the GMMs. Recent variants of SEDS, and in particular LPV-DS~\cite{figueroa2018physically,shavit2018learning}, attempt to relax the constraints by disconnecting the learning of the weighting function from the elementary DSs; LPV-DS still uses GMMs for learning the functions.

Recently, Bahl et al.~\cite{bahl2020neural} proposed a method to combine neural networks with DMPs. The main idea is to create a high-level controller with NNs that takes as input an unstructured state and \emph{selects} parameters of a DMP that acts as the low-level controller. Their method, called Neural Dynamic Policies (NDPs), was able to learn multiple LfD and RL scenarios effectively. In a recent extension~\cite{bahl2021hierarchical}, the authors provide a hierarchical formulation of their method that can be used to solve more complex tasks. To the best of our knowledge, this work proposes one of the first methods that effectively combine NNs with DSs and provide the first general-purpose policy (i.e., it can be used with almost any input and any robot) that is based on DSs. However, since their policy changes the dynamical system every $X$ steps there are still no theoretical guarantees for stability, but mostly rely on the data-driven capabilities of the NNs to capture this type of behaviors.

In this paper, we take inspiration from LPV-DS and NDPs and provide a policy structure, called Autonomous Neural Dynamic Policies (ANDPs), that (a) always produces asymptotically stable behaviors for the controllable part of the state, and that (b) is a general purpose policy that can work with any action space and can accept arbitrary inputs (e.g., images).

\section{Problem formulation}
\label{sec:problem_formulation}
\noindent We assume discrete-time dynamical systems that can be described by transition dynamics of the form:
\begin{align}
  \label{eq:dyn}
  \boldsymbol{x}^{t+1} = f(\boldsymbol{x}^t,\boldsymbol{u}^t) + \boldsymbol{w}
\end{align}
where the system is at state $\boldsymbol{x}^t\in\mathbb{R}^E$ at time $t$, takes control input $\boldsymbol{u}^t\in\mathbb{R}^U$ and ends up at state $\boldsymbol{x}^{t+1}$ at time $t+1$, $\boldsymbol{w}$ is i.i.d. Gaussian system noise, and $f$ is a function that describes the unknown transition dynamics.

We assume that the system is controlled through a parameterized \textit{policy} $\boldsymbol{u} = \pi(\boldsymbol{x}|\boldsymbol{\theta})$ that is followed for $M$ steps ($\boldsymbol{\theta}$ are the parameters of the policy). When following a particular policy for $M$ time-steps from an initial state distribution $p(\boldsymbol{x}^0)$, the system's states and actions jointly form \textit{trajectories} $\boldsymbol{\tau} = (\boldsymbol{x}^0,\boldsymbol{u}^0,\boldsymbol{x}^1,\boldsymbol{u}^1,\dots,\boldsymbol{x}^{M-1})$, which are often also called \textit{rollouts} or \textit{paths}.

In this work, we define a novel policy structure and learning procedure (called ANDPs) with stability guarantees. 
In an imitation learning scenario, we assume access to a few demonstrated trajectories $\{\boldsymbol{\tau}_i\}_{i=1,\dots,K}$, and we want to find the policy parameterization $\boldsymbol{\theta}$ that ``mimics'' the demonstrated trajectories as well as possible. In this work, we assume having access only to the states of the system, $\boldsymbol{x}^t$, and not of the control signals, $\boldsymbol{u}^t$. In other words, we have trajectories $\{\boldsymbol{s}_i\}_{i=1,\dots,K}$ of the form $\boldsymbol{s} = (\boldsymbol{x}^0,\boldsymbol{x}^1,\dots,\boldsymbol{x}^{M-1})$. This makes the problem slightly more difficult and usually enforces the use of a low-level controller ~\cite{osa2018algorithmic}.

\section{Proposed Policy Structure}
\label{sec:andps}
We make the assumption that the state of the system can be split into two parts: (a) a part that can be directly controlled (e.g., positions and velocities of the end-effector), and (b) a part that can only be observed and/or indirectly controlled (e.g., obstacles/objects). In particular (we omit the time notation, $t$, for clarity):
\begin{align}
\label{eq:state_vec}
  \boldsymbol{x} = \begin{bmatrix} \boldsymbol{x}_c\\\boldsymbol{x}_{nc}\end{bmatrix}\in\mathbb{R}^{d_c+d_{nc}},
\end{align}
where $\boldsymbol{x}_c$ is the part of the state that can be directly controlled and $\boldsymbol{x}_{nc}$ is the part of the state that can only be observed. $d_c$ and $d_{nc}$ are the state-space dimensions for the controllable and non-controllable parts, respectively ($d_c + d_{nc} = E$).

We define the control policy as a dynamical system with a fixed attractor $\boldsymbol{x}_c^*$ (formulated as a weighted sum of elementary linear dynamical systems):
\begin{equation}
  \label{eq:policy_dyn}
  \dot{\boldsymbol{x}}_c = \pi(\boldsymbol{x}) = 
  \sum_{i=1}^Nw_i(\boldsymbol{x})\boldsymbol{A}_i \Big(\boldsymbol{x}_c^* - \boldsymbol{x}_c\Big)
\end{equation}
where
$N$ is the number of elementary dynamical systems, $w_i(\boldsymbol{x})\in\mathbb{R}$ are state-dependent weighting functions, and
$\boldsymbol{A}_i\in\mathbb{R}^{d_c\times d_c}$, $\boldsymbol{x}_c^*\in\mathbb{R}^{d_c}$.

The control policy, $\pi(\boldsymbol{x})$ (Fig.~\ref{fig:policy_overview}), defines the desired velocity profile that the controllable state $\boldsymbol{x}_c$ should follow. Depending on the state representation one can directly use the output for commanding the robot, use a PD controller, or use some inverse dynamics/kinematics model. Note, that the controllable state $\boldsymbol{x}_c$ can also contain velocities (e.g., $\boldsymbol{x}_c = \{\boldsymbol{\xi}, \dot{\boldsymbol{\xi}}\}$, where $\boldsymbol{\xi}$ is the end-effector translation) and in that case the system is a second order DS. Although in this work we explore first-order DSs, our formulation allows for second-order DS systems.
\begin{theorem}
\label{eq:ds_theorem}
Assume that the controllable part of a state trajectory follows the policy as defined in Eq.~\ref{eq:policy_dyn}. Then, the function described by Eq.~\ref{eq:policy_dyn} is asymptotically stable to $\boldsymbol{x}_c^*$ if
\begin{equation}
    \label{eq:precond}
    \begin{cases}
      \boldsymbol{A}_i + \boldsymbol{A}_i^T \succ 0 & \text{the symmetric part of } A \text{ is psd}\\
      w_i(\boldsymbol{x}) > 0, & i=1,..,N, \forall\boldsymbol{x}\in\mathbb{R}^{E}
    \end{cases}       
\end{equation}
\end{theorem}
\begin{proof}
The proof follows classical Lyapunov analysis similar to~\cite{shavit2018learning}.
\end{proof}
The results of the above theorem can be described as \emph{``The controllable part of the system will always converge to the fixed attractor $\boldsymbol{x}_c^*$''}. Although this does not guarantee that the whole system state will converge to a desired state, this is an important property for a policy to have, as it will always generate commands that will eventually drive the controllable part of the system to a stable point. It is important to note that this property holds if the controllable system can follow the commanded velocities perfectly, and does not take into account the properties of a possible low-level controller (e.g., a controller based on the pseudoinverse of the Jacobian for end-effector control) or the rest of the environment. This is, however, common in the LfD literature since designing a policy that can guarantee the stability of the whole system and take into account the properties of the low-level controller is a challenging task and would require bulk approximations to be made~\cite{chatzilygeroudis2019survey,osa2018algorithmic}. Nevertheless, in all of our experiments, we never observed diverging motions and we did not have to tune the low-level controllers to avoid such situations. Overall, these limitations do not seem to have a big impact on the resulting behaviors and we were able to learn a wide range of motions using different low-level controllers (in joint- and task-space).
\vspace{-0.5em}

\subsection{ANDPs via Neural Networks}
The main intuition of ANDPs is to combine the power of neural networks to learn from data while keeping the stability guarantees of the traditional DS-based policies. In order to be able to represent the ANDPs (Eq.~\ref{eq:policy_dyn}) with a neural network, we have to: (a) find the ``learnable'' parameters, (b) make sure that the policy is fully differentiable, and (c) handle the constraints of Eq.~\ref{eq:precond} properly.

In order to define the learnable parameters, we need to identify parameters of Eq.~\ref{eq:policy_dyn}. First, the matrices $\boldsymbol{A}_i$ do not depend on the state, and thus we can directly optimize for their parameters. Second, in order to define each $w_i(\boldsymbol{x})$, we use one neural network for all of them. More concretely, we define a neural network $\Psi$ which takes input a full system state $\boldsymbol{x}$ and predicts a vector $\boldsymbol{W}\in\mathbb{R}^N$. In other words, $\boldsymbol{W} = \Psi(\boldsymbol{x}|\boldsymbol{\psi})$, where $\boldsymbol{\psi}$ are the parameters of the neural network. The $i\text{-th}$ element of the $\boldsymbol{W}$ vector represents $w_i(\boldsymbol{x})$. So, the total learnable parameters of the policy are $\boldsymbol{\theta} = \{\boldsymbol{A}_1,\dots,\boldsymbol{A}_N,\boldsymbol{\psi}\}$. It is easy to see that since each $\boldsymbol{A}_i$ is a simple matrix, and $\boldsymbol{\psi}$ parameters of a neural network, the whole policy is fully differentiable.

One can also add the attractor point, $\boldsymbol{x}^*_c$, to the optimization variables. The policy would still be differentiable ($\boldsymbol{x}^*_c$ is a free parameter). In our experiments, we observed that in IL/LfD scenarios, that would require a complicated loss function, and thus we left this exploration for future work. In this paper, we assume a fixed attractor that is equal to the average of the last points of all the demonstrated trajectories.

The last element is to be able to optimize the parameters $\boldsymbol{\theta}$ given some objective function, while respecting the constraints defined in Eq.~\ref{eq:precond}. In the general case, this would require a constrained optimization problem to be performed, but this can be challenging to do for high-dimensional parameter spaces, such as the ones generated by neural networks. In order to bypass this issue but still respect the constraints, we perform the following steps:
\begin{itemize}
    \setlength\itemsep{0em}
    \item First, each $w_i$ should be positive. We can easily generate positive numbers by adding an $exp$ layer after the last layer of $\Psi$. In this work, we always use a \emph{softmax} layer as the last layer of $\Psi$ that generates positive values that sum to one; it also makes more sense as we want to combine the elementary DSs that $\boldsymbol{A}_i$ define.
    \item We, now, need to satisfy the constraint $\boldsymbol{A}_i + \boldsymbol{A}_i^T \succ 0$. If we assume real matrices and define $\boldsymbol{A}_i = \boldsymbol{B}_i + \boldsymbol{C}_i - \boldsymbol{C}_i^T$, where $\boldsymbol{B}_i$ is a symmetric positive definite matrix (no restrictions for $\boldsymbol{C}_i$), we can see that $\boldsymbol{A}_i + \boldsymbol{A}_i^T = \boldsymbol{B}_i + \boldsymbol{C}_i - \boldsymbol{C}_i^T + \boldsymbol{B}_i + \boldsymbol{C}_i^T - \boldsymbol{C}_i = 2\boldsymbol{B}_i \succ 0$. $\boldsymbol{B}_i$ is symmetric, thus $\boldsymbol{B}_i = \boldsymbol{B}_i^T$, and $\boldsymbol{C}_i - \boldsymbol{C}_i^T$ defines a skew symmetric matrix. This gives us the ability to freely optimize for the parameters of $\boldsymbol{C}_i$. So, we are left with handling the case of optimizing for a symmetric positive definite matrix, $\boldsymbol{B}_i$.
    \item If we assume real matrices, a symmetric positive definite matrix $\boldsymbol{B}_i$ can be factorized as $\boldsymbol{B}_i = \boldsymbol{L}_i\boldsymbol{L}_i^T$ (Cholesky decomposition), where $\boldsymbol{L}_i$ is a lower triangular matrix with real and positive diagonal entries. It is easy to see that we can optimize for the parameters of $\boldsymbol{L}_i$ and reconstruct $\boldsymbol{B}_i$ that we need.
\end{itemize}
Using the above steps we can now perform unconstrained optimization while still ensuring that the constraints in Eq.~\ref{eq:precond} are fulfilled. This is important as we are more confident that the optimization will converge to good solutions.

\section{Experiments}

\begin{figure*}[!htb]
    \centerline{\includegraphics[width=\linewidth]{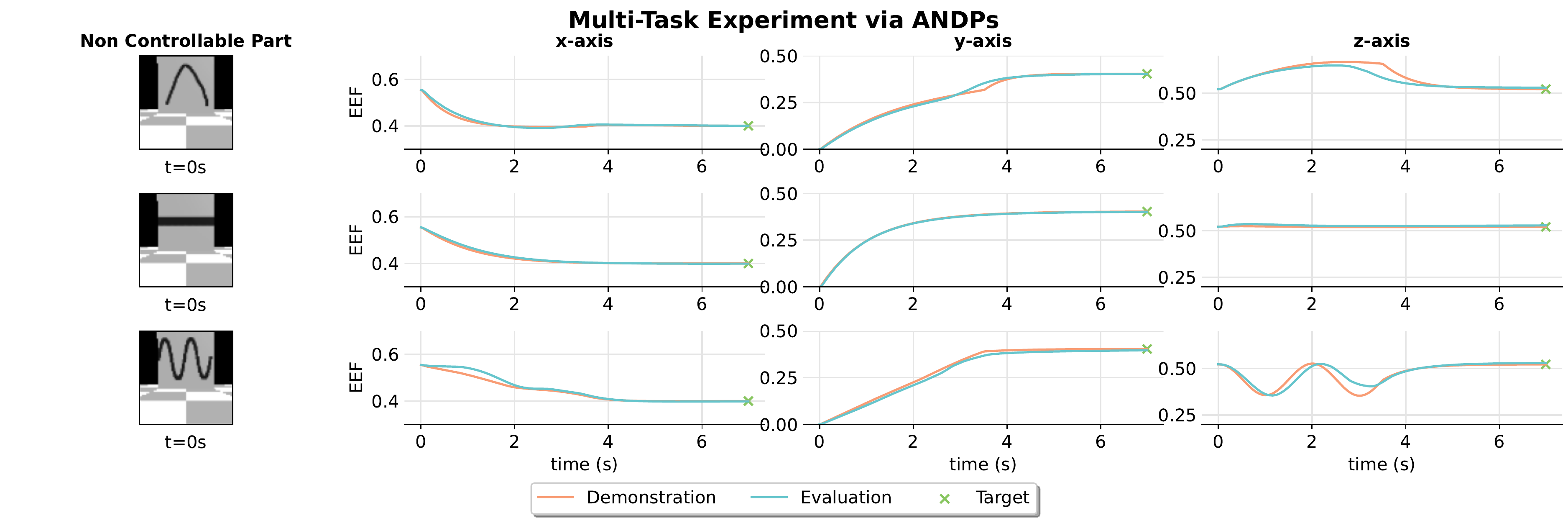}}
    \caption{Multi-task scenario with image inputs. All tasks are learned with a single model that can distinguish between tasks given an image input.}
    \label{fig:andps_images}
    \vspace{-2em}
\end{figure*}

In this paper, we focus on IL/LfD scenarios. Through the conducted experiments,
we attempt to answer the following questions:
\begin{enumerate}
    \setlength\itemsep{0em}
    \item Do ANDPs produce stable behaviors? How well can they reproduce the initial demonstrations?
    \item Can ANDPs learn complex movements for a realistic robotic task? Are they robust? 
    \item Can ANDPs accept arbitrary inputs like 3D orientations? Can they accept even raw images?
    \item Can ANDPs work on a physical robot?
\end{enumerate}

In order to answer the first three questions, we devise a multi-task scenario where the goal is for ANDPs to learn multiple tasks into one policy; we use the DART open-source simulator~\cite{lee2018dart}. The idea is to use the non-controllable part of the state $\boldsymbol{x}_{nc}$ to "define" which task we want the robot to perform. So for each task, $\boldsymbol{x}_{nc}$ is an image captured with a camera that is mounted to the robot's end-effector and points directly to a sign that displays the picture that corresponds to the particular movement. We also test the reactiveness of ANDPs by changing the image in the middle of the evaluation pipeline, and the robustness of the learned policies by inserting force perturbations.

In order to answer the fourth question, we devise the following experiment: \emph{We use a physical Franka Panda robot to collect three demonstrations with kinesthetic guidance for a pouring task and learn a policy with data collected from a physical setup.} In this task the robot needs to pour liquid from one cup into a bowl and we control the robot in end-effector space with changing orientation.

\subsection{Multi-Task Learning}
In this section, we want to determine whether ANDPs have the capability of learning intricate 3D motions, demonstrate the flexibility of ANDPs compared to traditional LfD methods, and exhibit the reactive and resilient  nature of the learned policy against perturbations. We collect one demonstration for each of the following movements: a sinusoidal motion, a linear motion, and  a curvilinear motion, so that we can devise a multi-task scenario where the goal is for ANDPS to learn multiple tasks into one policy. In essence, we use the non-controllable part of the state to "define" which task we want the robot to perform. In order to create the labels, we simulate a sign with the image corresponding to every motion across the robot, we then attach a camera to the end-effector of the arm facing towards the sign and we shoot a grayscale image for every sample. We take all 3 demonstrations and we have a state of the form: $\boldsymbol{x} = \{\boldsymbol{x}_{nc}, \boldsymbol{x}_c\} = \{\mathcal{I}, x, y\}$, where $\mathcal{I}\in\mathbb{R}^{64\times 64}$ is a grayscale image. We use a Convolutional Neural Network (CNN) to model the weight function $\Psi$; in particular, we use a LeNet~\cite{lecun1989backpropagation} variation. We learn one model for all three tasks. In Fig.~\ref{fig:andps_images}, we see that ANDPs are able to learn to distinguish the three tasks while always ensuring convergence to the fixed attractor. Moreover, the learned policy is reactive and robust to perturbations. To showcase the broader concept of reactiveness we start an evaluation run with the image corresponding to the linear movement displayed, at $t=1s$ we switch the displayed image  to the one representing the sinusoidal movement. We observe that the robot changed its motion to follow the shape of the corresponding movement (Fig.~\ref{fig:andps_reactive_1}). 

\begin{figure}[!h]
    \centerline{\includegraphics[width=\linewidth]{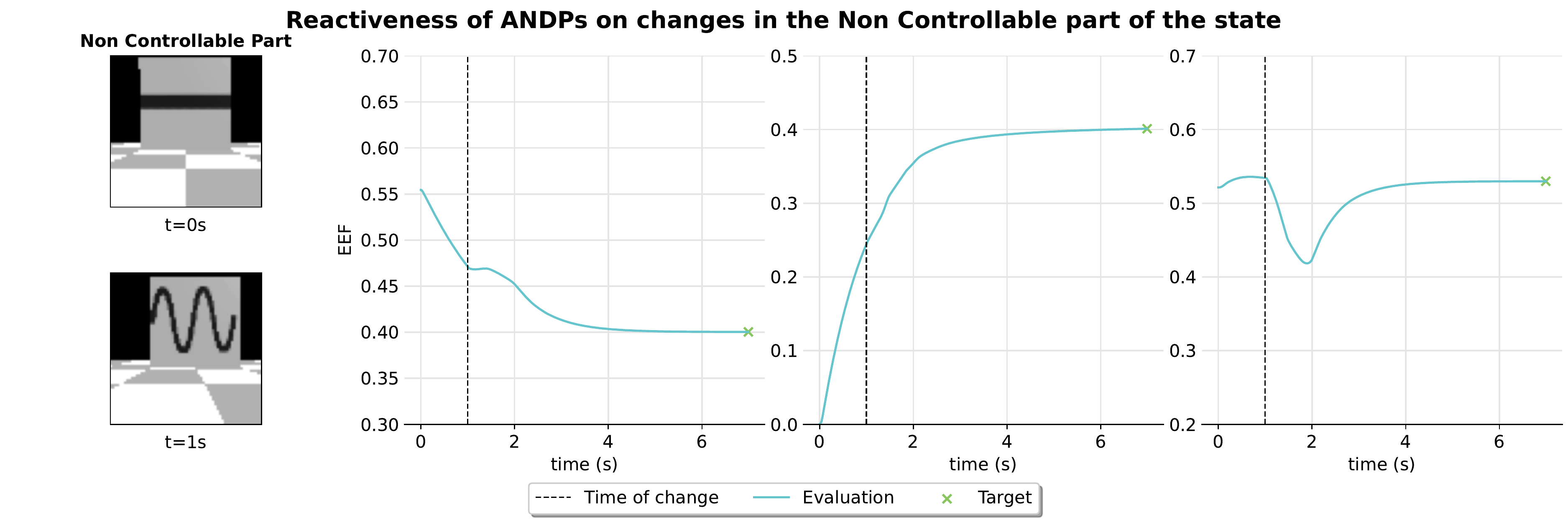}}
    \caption{Reactiveness of ANDPs on changes in the non-controllable part of the state, we switch from the line image to the sine image at t=$1s$.}
    \label{fig:andps_reactive_1}
    \vspace{-1em}
\end{figure}

To show that ANDPs are robust and reactive to force perturbations we apply an external force to the robot twice during the execution: once at the beginning of the behavior, and once at $t=5\,s$. We observe that the robot is able to converge to the attractor and follow the overall shape of the behavior (Fig.~\ref{fig:andps_reactive_2}). 

\begin{figure}[!h]
    \centerline{\includegraphics[width=\linewidth]{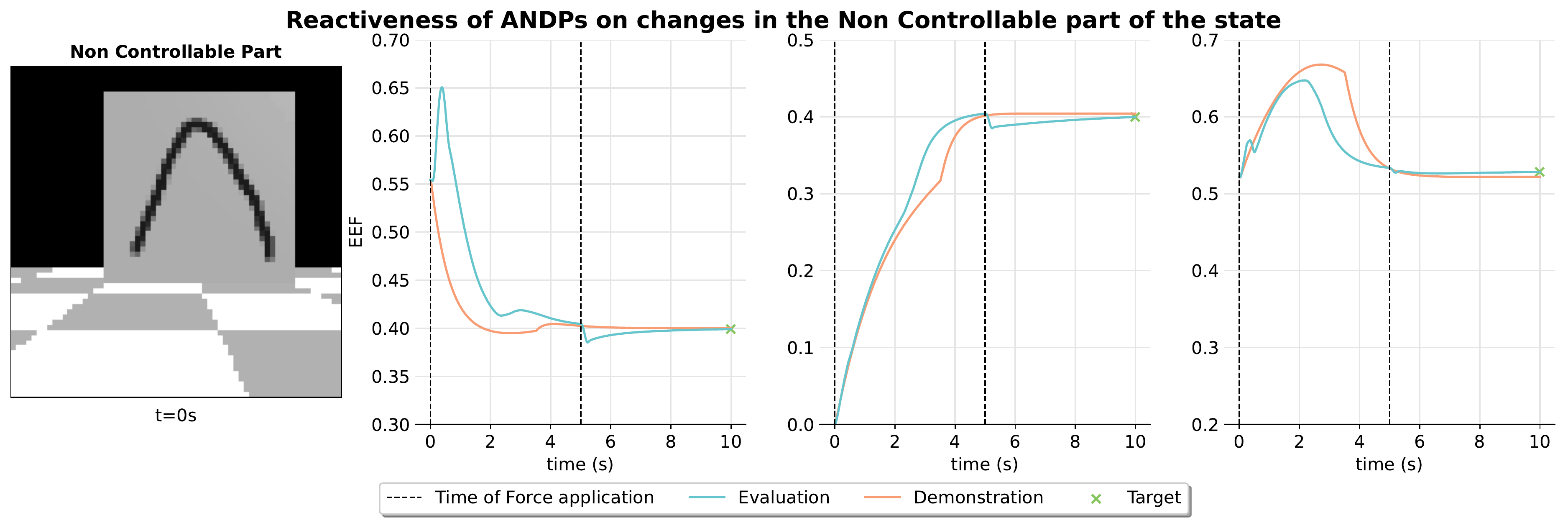}}
    \caption{Reactiveness of ANDPs on external force perturbations we apply an external force twice, one at $t=0s$ and $t=5s$.}
    \label{fig:andps_reactive_2}
    \vspace{-1em}
\end{figure}

\subsection{Physical Robot Experiment}
In this section, we want to identify whether ANDPs: (a) work with realistic demonstrations, and (b) can learn a task that requires precision and end-effector orientation control. For those reasons, \textbf{we collect three demonstrations with kinesthetic guidance on the physical robot performing a pouring task}: the robot holds a cup filled with liquid and needs to pour it inside a bowl (Fig.~\ref{fig:pour_demo}). For safety, we ``emulate'' the liquid with small plastic objects. We then learn a policy with ANDPs using the collected demonstrations with a state of the form $\boldsymbol{x}\equiv\boldsymbol{x}_c = \{x, y, z, r_x, r_y, r_z\}$, where $\{x,y,z\}$ is the end-effector translation and $\{r_x, r_y, r_z\}$ is the end-effector orientation expressed in Euler XYZ angles.

\begin{figure}[!h]
    \centerline{\includegraphics[width=\linewidth]{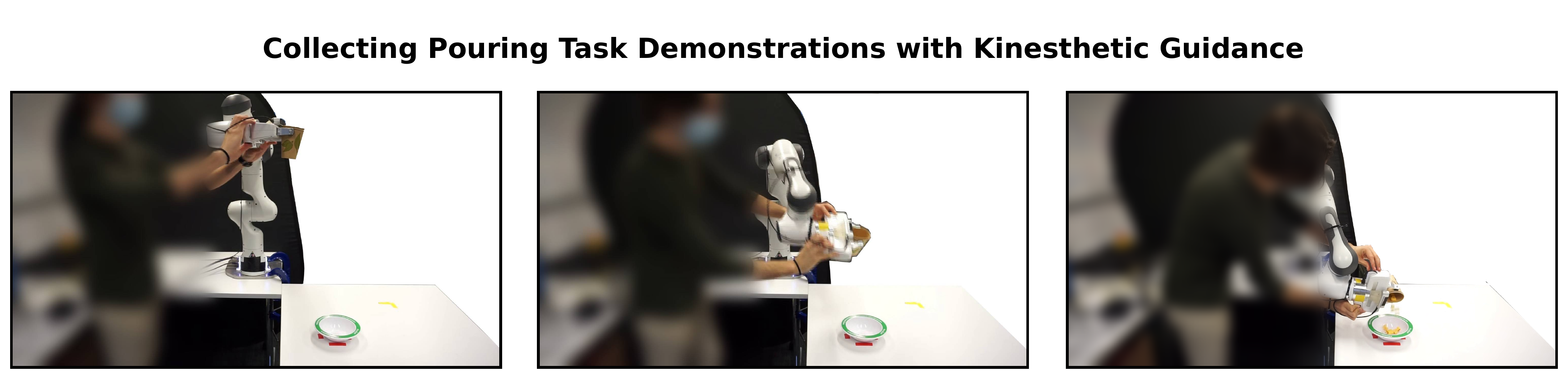}}
    \caption{ Collecting pouring task demonstrations via kinesthetic guidance on the Franka Emika Panda robot.}
    \label{fig:pour_demo}
    \vspace{-1em}
\end{figure}

The results showcase that ANDPs work reliably in this setting and the robot successfully pours the liquid from the cup to the bowl (Fig.~\ref{fig:pouring_task_real}). In order to validate more thoroughly the effectiveness of the learned policy, we perform \emph{10 replicates with different initial configurations} of the robot and measure \emph{the percentage of the plastic objects that end up inside the bowl} (Fig.~\ref{fig:results_pour} (a)).  We get a median percentage of $100\%$ with $67.5\%$ and $100\%$ for the 25-th and 75-th percentiles respectively (Fig.~\ref{fig:results_pour} (b)).


\begin{figure}[ht!]
    \includegraphics[width=\linewidth]{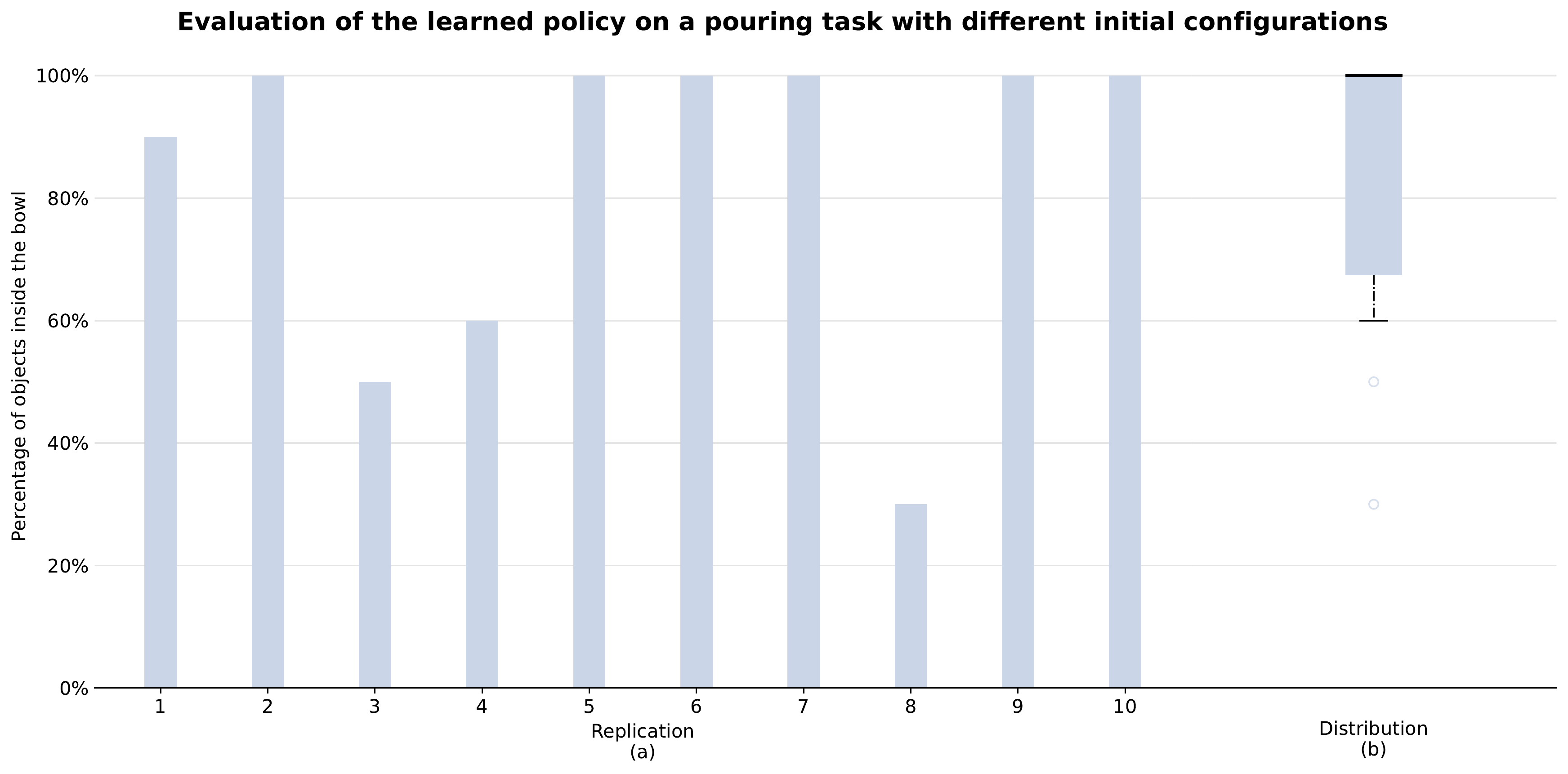}
    \caption{Percentage of objects that ended up inside the bowl on the 10 replications of the pouring task.}
    \label{fig:results_pour}
    \vspace{-1em}
\end{figure}
\begin{figure}[ht!]
    \centering
    \includegraphics[clip,trim=20 0 20 20,width=0.24\linewidth]{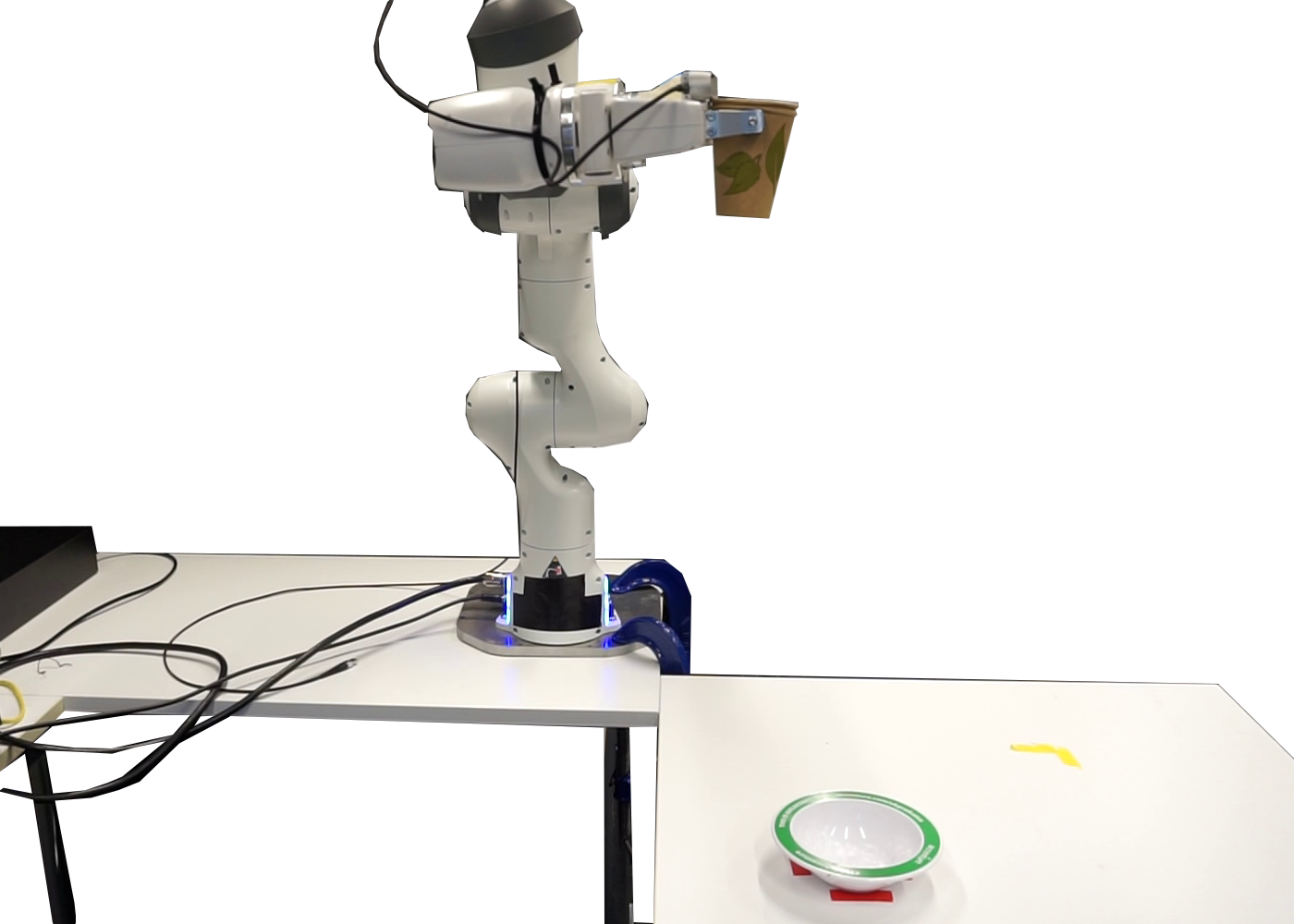}
    \includegraphics[clip,trim=20 0 20 20,width=0.24\linewidth]{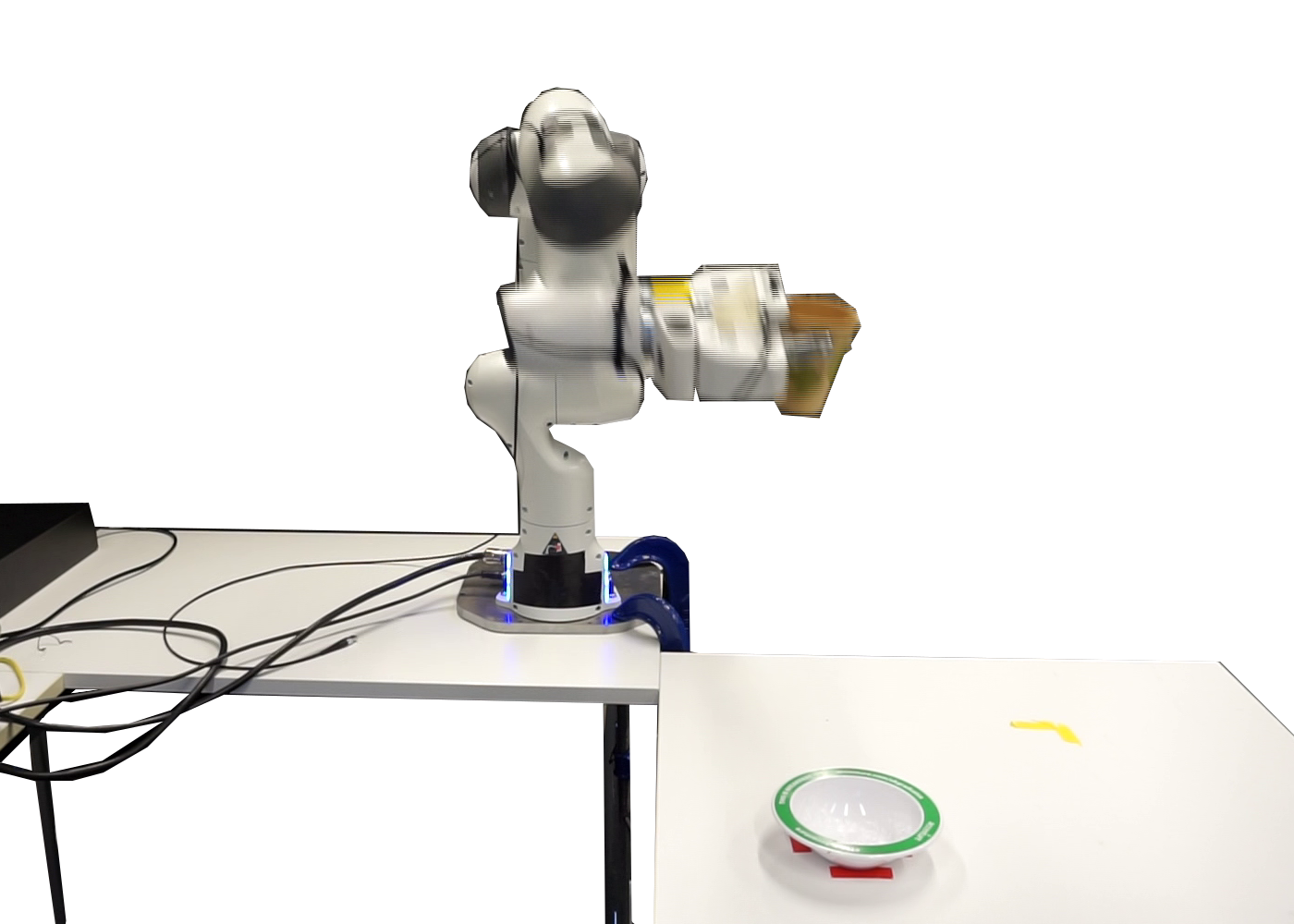}
    \includegraphics[clip,trim=20 0 20 20,width=0.24\linewidth]{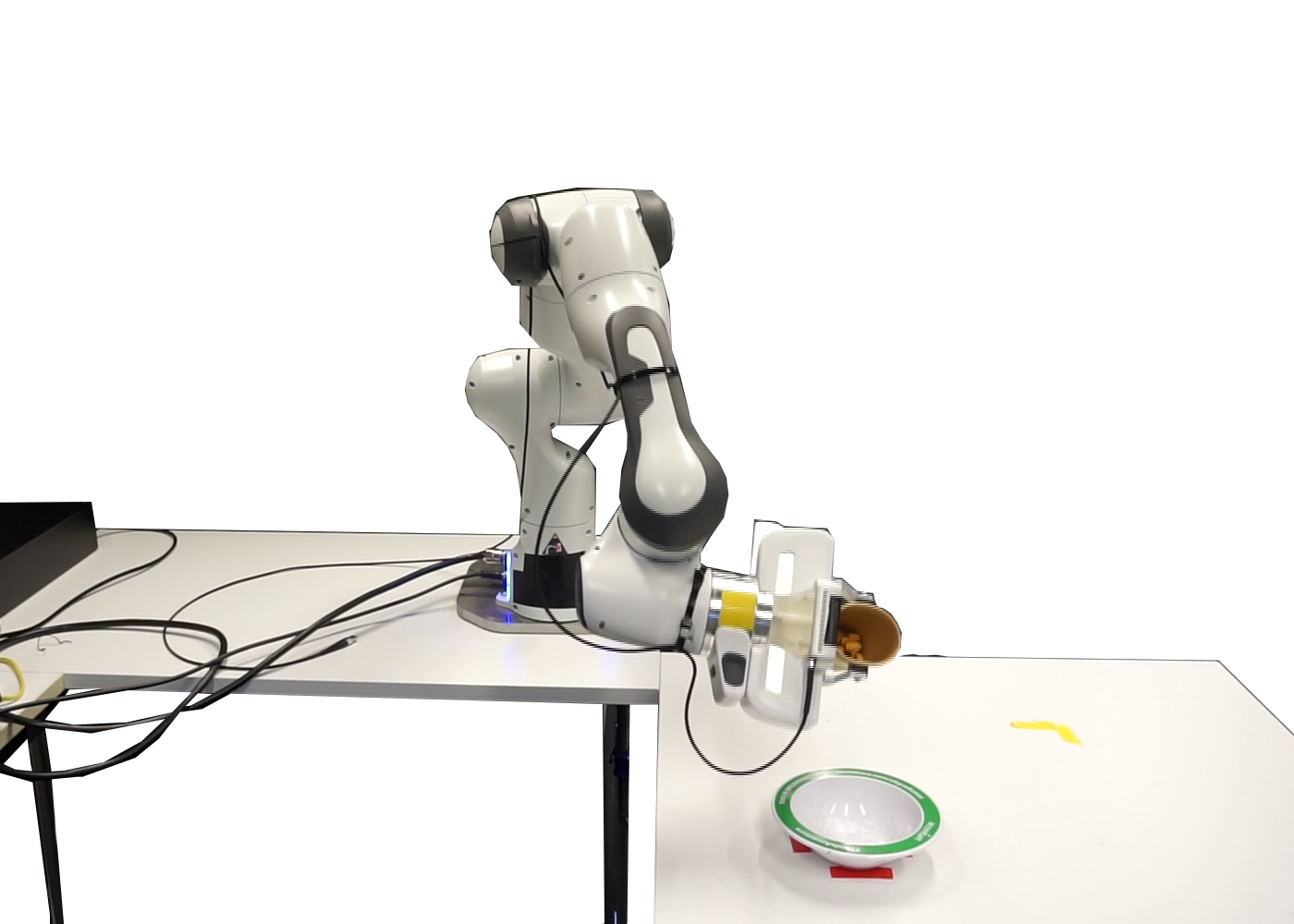}
    \includegraphics[clip,trim=20 0 20 20,width=0.24\linewidth]{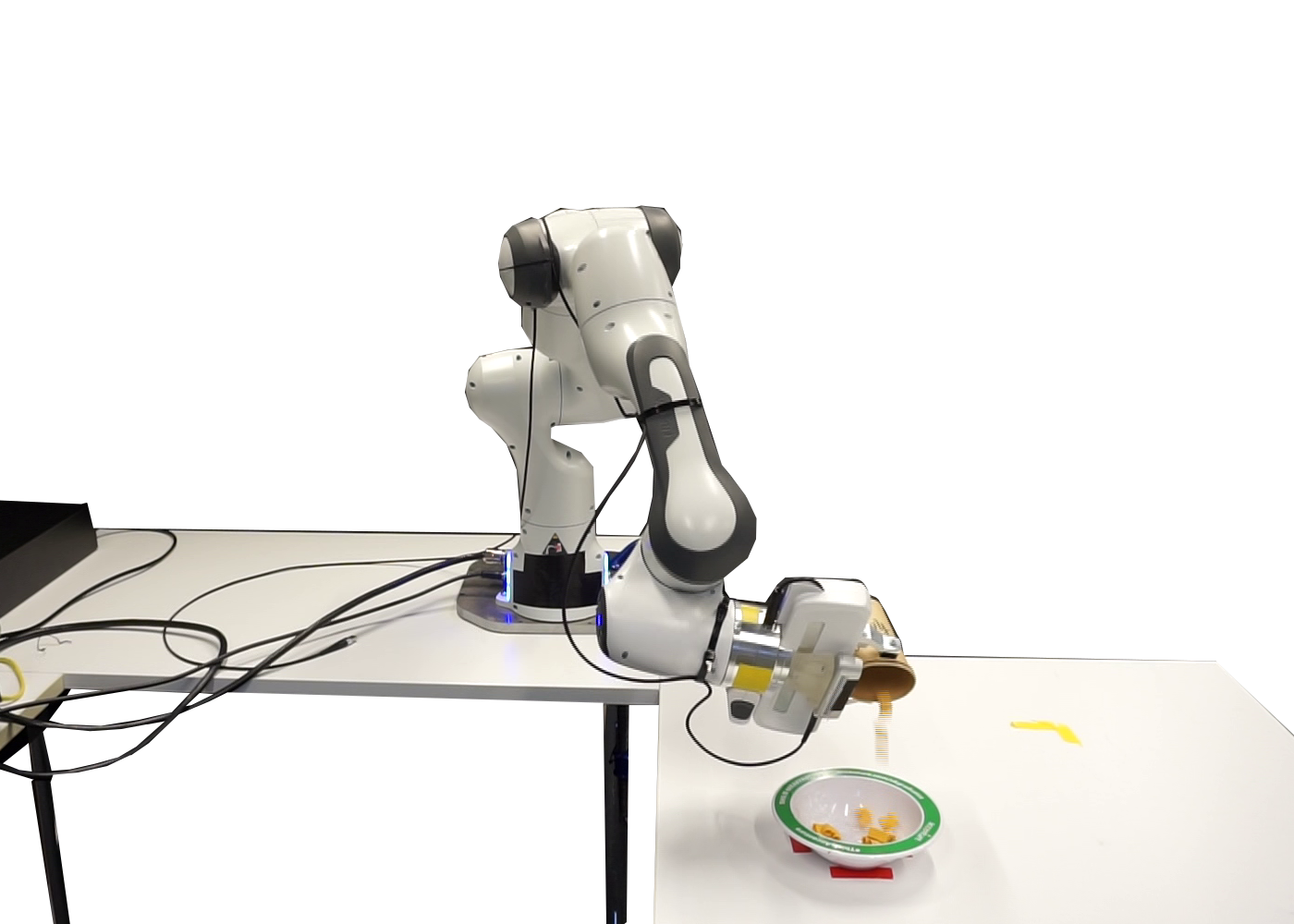}
    \caption{From left to right, screenshots of a successful trial of the pouring task in the physical setting.}
    \label{fig:pouring_task_real}
    \vspace{-1em}
\end{figure}

\section{Results and Conclusions}
ANDPs are one of the first policy structures for robot learning that are general purpose while ensuring asymptotic stability of the produced behaviors (at least for the controllable part, i.e., the robot). Using ANDPs we were able to learn many different tasks with different action space parameterizations and different input types. Although we performed experiments only in IL/LfD scenarios, ANDPs are fully differentiable and generic policies and thus can also be used in pure RL settings. We will explore this property of ANDPs in future work.

Another important feature of ANDPs is their inherit explainability. Since the underlying policy is a sum of elementary linear DSs, one can examine the ANDPs via more classical tools for understanding the reasoning of the policy behind its decisions. We aim at investigating this in detail in future work.

The main limitation of ANDPs' current formulation is the need for having a fixed attractor (even if the attractor is ``learned'', it is still a fixed attractor, that is, it does not move throughout the episode/experience). This has two important consequences: (a) it might be difficult for the policy to learn long-horizon complicated tasks, and (b) we need to find a more complicated optimization scheme in order to relax the constraints for the elementary DSs and allow non-monotonic motions (i.e., motions that ``go away'' from the attractor). In future work, we aim at exploring the possibility of having a moving attractor while keeping the stability properties.

Finally, although we provided interesting results with image inputs, we aim at performing extensive experiments with different tasks to further validate the effectiveness of ANDPs in these settings.
\section*{Acknowledgments}

Konstantinos Chatzilygeroudis was supported by the Hellenic Foundation for Research and Innovation (H.F.R.I.) under the ``3rd Call for H.F.R.I. Research Projects to support Post-Doctoral Researchers'' (Project Acronym: NOSALRO, Project Number: 7541). Dimitrios Kanoulas and Valerio Modugno were supported by the UKRI Future Leaders Fellowship [MR/V025333/1] (RoboHike).

\bibliographystyle{IEEEtran}
\bibliography{references}

\end{document}